\documentclass{article}

\PassOptionsToPackage{numbers,compress}{natbib}
\usepackage[preprint]{neurips_2026}

\usepackage[utf8]{inputenc}
\usepackage[T1]{fontenc}
\usepackage{microtype}
\usepackage{graphicx}
\usepackage{subcaption}
\usepackage{booktabs}
\usepackage{longtable}
\usepackage{amsmath}
\usepackage{amssymb}
\usepackage{mathtools}
\usepackage{amsthm}
\usepackage{bm}
\usepackage{enumitem}
\usepackage{hyperref}
\usepackage{url}
\usepackage[capitalize,noabbrev]{cleveref}
\usepackage{algorithm}
\usepackage{algorithmic}
\usepackage{xcolor}

\graphicspath{{./}{uploads/}}

\newcommand{\yrcite}[1]{\citeyearpar{#1}}

\newcommand{\bs}{\mathbf{s}}
\newcommand{\ba}{\mathbf{a}}
\newcommand{\bx}{\mathbf{x}}
\newcommand{\by}{\mathbf{y}}

\DeclareMathOperator{\softmax}{softmax}

\theoremstyle{plain}
\newtheorem{theorem}{Theorem}[section]
\newtheorem{proposition}[theorem]{Proposition}

\theoremstyle{definition}

\theoremstyle{remark}

\title{Self-Guided Process Reward Optimization with Redefined Step-wise Advantage for Process Reinforcement Learning}

\author{
Fei Wu \\
HKUST(GZ)
\and
Shuxian Liang \\
Tongji Univ.
\and
Yibo Yang \\
SJTU
\and
Yang Lin \\
PKU
\AND
\normalfont
Jing Tang \\
HKUST(GZ)
\and
Lei Chen \\
HKUST(GZ)
\and
Xiansheng Hua \\
Tongji Univ.
\and
Hao Kong\thanks{Corresponding author. Please contact us at: \texttt{fwu750@connect.hkust-gz.edu.cn} or \texttt{konghao@pku.edu.cn}.} \\
Tencent
}
\begin{document}

\maketitle
\begin{abstract}

Process Reinforcement Learning~(PRL) has demonstrated considerable potential in enhancing the reasoning capabilities of Large Language Models~(LLMs). 
However, introducing additional process reward models incurs substantial computational overhead, and there is no unified theoretical framework for process-level advantage estimation.
To bridge this gap, we propose \textbf{S}elf-Guided \textbf{P}rocess \textbf{R}eward \textbf{O}ptimization~(\textbf{SPRO}), a novel framework that enables process-aware RL through two key innovations: (1) we show that process rewards can be derived intrinsically from the policy model itself, and (2) we redefine step-wise advantage by introducing well-defined Cumulative Process Rewards~(\textbf{CPR}) and \textbf{M}asked \textbf{S}tep \textbf{A}dvantage~(\textbf{MSA}), which facilitates rigorous step-wise action advantage estimation within shared-prompt sampling groups.
Our experimental results show that SPRO outperforms vanilla GRPO with 3.4× higher training efficiency and a 12.9\% test accuracy improvement. 
Furthermore, SPRO maintains a stable and elevated policy entropy throughout training while achieving a considerable reduction in response length, evidencing sufficient exploration and prevention of reward hacking.
Notably, SPRO incurs no additional computational overhead compared to outcome-supervised RL methods such as GRPO, which benefit industrial implementation.

\end{abstract}

\section{Introduction}
\label{sec:intro}

Reasoning ability is fundamental to the intelligence of language models and plays a pivotal role in advancing artificial general intelligence~(AGI).
Recent open source reasoning models, such as OpenAI's o1~\cite{openai2024o1} and DeepSeek's R1~\cite{guo2025deepseek}, demonstrate the effectiveness of reinforcement learning in reasoning tasks. 
Most current RL algorithms~\cite{guo2025deepseek, shao2024deepseekmath, team2025kimi, ahmadian2024rloo} optimize policy models based exclusively on outcome rewards, neglecting intermediate feedback. This sparse reward paradigm leads to inefficient learning~\cite{qu2025survey, cui2025entropy}, highlighting the need for efficient and scalable process-based reinforcement learning algorithms.
This raises a fundamental research question: \emph{how to provide well-defined process rewards and guide the policy optimization effectively?}

Although significant research~\cite{lightman2023let,feng2023alphazero,snell2024scaling,wang2024math} has been devoted to addressing this question, there is no unified theoretical framework for process-level advantage estimation. 
Training auxiliary process reward models~(PRMs), which estimate the future success of intermediate steps, is a commonly adopted strategy. 
However, PRMs exhibit several widely recognized practical limitations:

\textbf{(1) Difficult to train:} 
    Human-annotated process-level labels lack scalability~\citep{lightman2023let}, while automatic annotation often fails to provide reliable supervision~\citep{guo2025deepseek}.
\label{challenge:1}

\textbf{(2) High computational cost:} 
~    In contrast to the widely adopted dual-model framework~(policy and reference model) in outcome-supervised algorithms such as~\citet{guo2025deepseek},~\citet{ shao2024deepseekmath} and~\citet{ahmadian2024rloo}, which significantly contributes to scalability and industrial adoption, PRM-based methods introduce an auxiliary reward model. 
    Loading of the additional model requires considerable GPU memory allocation, which constrains the batch size and substantially degrades training throughput and efficiency.
\label{challenge:2}

\textbf{(3) Non-scalable utilization:} 
    Existing methods typically leverage PRMs to rerank candidate responses~\cite{uesato2022solving} or perform Monte Carlo Tree Search~(MCTS)~\cite{lightman2023let, feng2023alphazero, wang2024math}, aiming to improve reasoning trajectories. 
    However, constructing the reasoning search space requires sequential rollouts at each step~\cite{snell2024scaling}, making these approaches non-scalable in online RL.

Recent works have proposed novel methods for acquiring high-quality PRMs.~\citet{rafailovrfrom} demonstrate that a well-trained DPO model can inherently achieve credit assignment, effectively expressing token-level rewards in the Markov Decision Process~(MDP) framework of LLMs.
Extending DPO framework,~\citet{yuan2024free} introduce a more generalized implicit PRM training paradigm that replaces preference pairs with point-wise labeled trajectories, which can be trained using 
cross-entropy loss. Subsequently,~\citet{cui2025process} improve the implicit PRM methodology by proposing the PRIME framework, which effectively combines token-level rewards with outcome rewards to calculate trajectory advantages.
The 
derived token-level reward functions proposed in~\citet{rafailovrfrom},~\citet{yuan2024free}, and~\citet{cui2025process} eliminate the need for explicit process annotations, thus addressing the practical limitation (1) and streamlining the training pipeline.

However, PRIME~\cite{cui2025process} relies on an auxiliary reward model $\pi_{\varphi}$ to parameterize implicit PRM, which inherits the practical limitation (2) and requires iterative training updates throughout the optimization process. This approach not only consumes additional GPU memory but also introduces non-negligible computational overhead. Additionally, although PRIME~\citep{cui2025process} circumvents limitation (3) by estimating advantages via a Monte Carlo estimator with a leave-one-out baseline, it aggregates all process rewards into a single group for normalization (see~\cref{fig:Compare-frameworks}). 
It deviates from standard advantage-based policy gradient methods like PPO~\cite{schulman2017proximal}, thereby introducing significant estimation bias.

%
To address the computational inefficiency of auxiliary PRMs in industry-scale PRL while 
achieving a more reasonable process advantage estimation, in this paper, we propose Self-guided Process Reward Optimization~(SPRO), a PRM-free algorithm for process reinforcement learning as shown in~\cref{fig:Illustration_of_SPRO}. We demonstrate that process rewards can be self-guided directly from the policy model itself, as SPRO eliminates both the annotation requirements and computational overhead inherent to PRM-based approaches and preserves the simplicity and scalability of outcome-supervised RL algorithms~\cite{shao2024deepseekmath,ahmadian2024rloo}, which benefit the industrial implementation. 

Moreover, SPRO offers a theoretical framework for step-level advantage estimation by redefining the step-wise advantage through Cumulative Process Reward~(CPR). 
%
This approach aligns with the classic advantage-based policy gradient framework by leveraging the nature of masked attention, which 
encodes prefix-sequence information. Specifically, CPR implicitly aggregates the process rewards from all preceding steps in the prefix sequence as a surrogate for process rewards, enabling more accurate expected return estimation at each timestep.
For advantage estimation, we extend the formulation of group-relative advantage from outcome-supervised algorithms~\cite{shao2024deepseekmath,ahmadian2024rloo} and investigate the estimation of step-level advantage.
To enable fair comparisons, we introduce Masked Step Advantage~(MSA), which enforces strict per-step comparison within shared-prompt sampling groups.

\begin{figure*}[t]
    \centering
    \begin{subfigure}[b]{0.49\textwidth} 
        \centering
        \includegraphics[width=\textwidth]{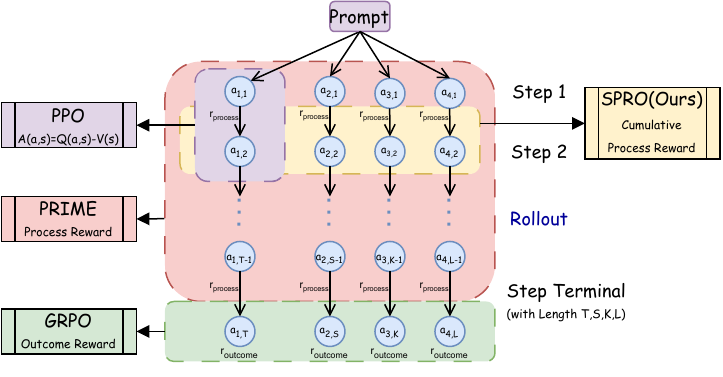}
        \caption{\textbf{The directions of advantage calculation across different framework.}}
        \label{fig:Compare-frameworks}
    \end{subfigure}
    \begin{subfigure}[b]{0.49\textwidth} 
        \centering
        \includegraphics[width=\textwidth]{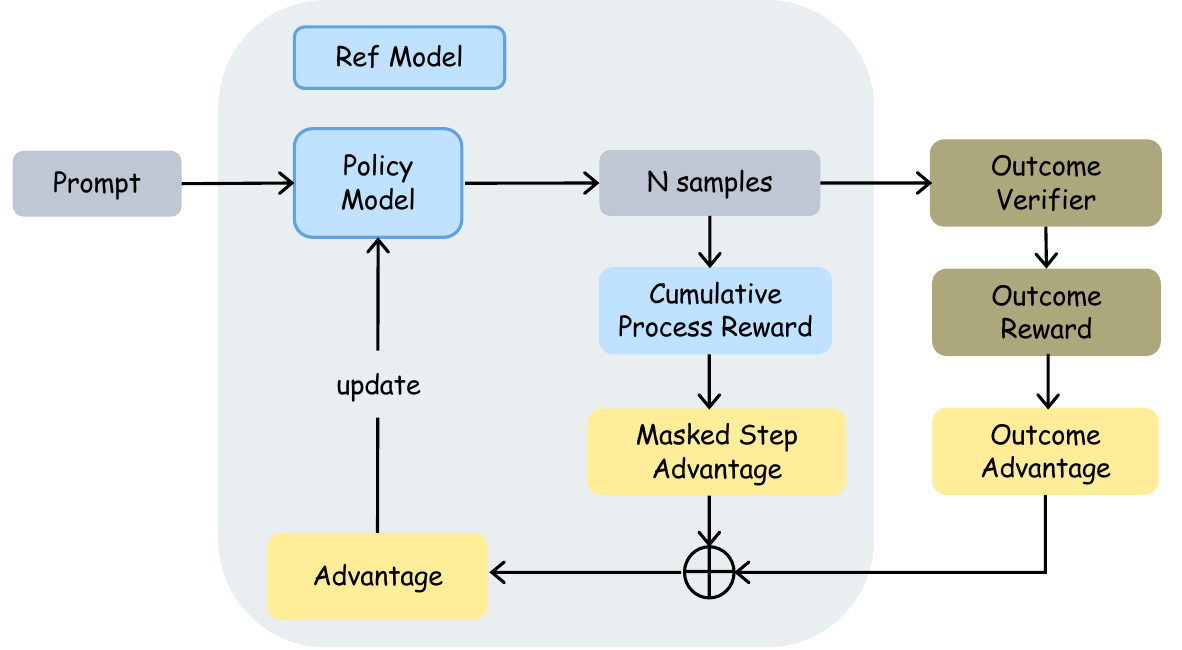}
        \caption{\textbf{Illustration of SPRO.}}
        \label{fig:Illustration_of_SPRO}
    \end{subfigure}
    \caption{(a): PPO estimates advantages based on single-step state transitions. GRPO treats all terminal states as one-step transitions from the initial prompt state. PRIME aggregates all process rewards into a single group for normalization. In contrast, we propose SPRO, which groups rewards at the same step for calculation. (b): SPRO utilizes Cumulative Process Rewards directly derived from the policy model, thereby eliminating the need for an additional process reward model and establishing a dual-model framework comparable to outcome-supervised approaches.}
    \label{fig:intro_second}
    \vspace{-1em}
\end{figure*}

As shown in~\cref{fig:Compare-frameworks}, we compare advantage functions across mainstream methods. Assume that four responses are sampled, where each response receives a outcome reward, and each intermediate step is assigned a process reward. 
GRPO~\cite{guo2025deepseek} estimates policy gradients by computing relative advantages within \emph{trajectory groups} using outcome rewards, while 
PRIME~\cite{cui2025process} adapts the grouping paradigm to all the token-level rewards for normalization.
As for SPRO, we employ the Cumulative Process Reward~(CPR) to compute the step-wise rewards at step $t$.  
Subsequently, for each identical step across different trajectories, group-wise normalization is applied to the rewards to obtain the Masked Step Advantage~(MSA).

The experimental results demonstrate significant improvements of SPRO over baseline methods. SPRO achieves $12.9\%$ higher test accuracy than vanilla GRPO and $7.2\%$ higher than PRIME, while reducing computational costs to $29\%$~(vs. GRPO) and $15\%$~(vs. PRIME) of GPU hours for equivalent performance.
The comparisons on response length and policy entropy also demonstrate that our approach simultaneously addresses two long-standing challenges that have attracted significant community attention: (1) improving token efficiency in reasoning~\cite{qu2025survey,liu2025understanding}, and (2) mitigating policy entropy collapse~(or reward hacking) during RL training~\cite{cui2025entropy}. 
These phenomena are discussed in~\cref{sec:analysis}. The dual improvement indicates that our framework enables the policy model to more effectively recognize 
the advantages of each step, resulting in both 
efficient reasoning and effective action space exploration.

The main contributions are summarized as follows: 
\begin{itemize}
    \item We introduce a novel RL framework for LLMs, \emph{Self-Guided Process Reward Optimization~(SPRO)}, which eliminates the need for costly PRMs and retains the same simplicity and scalability as outcome-supervised RL.
    \item We redefine the step-level advantage by introducing \emph{Cumulative Process Reward~(CPR)} as a surrogate for self-guided process rewards and further propose \emph{Masked Step Advantage~(MSA)}, which enables a strict per-step comparison within shared-prompt sampling groups to estimate step-level advantages.
    \item Our experimental results demonstrate that SPRO simultaneously improves accuracy and training efficiency while resolving two critical challenges: token efficiency and policy entropy collapse. SPRO significantly reduces the length of reasoning sequences while achieving higher accuracy. Moreover, SPRO maintains higher policy entropy, promoting more efficient exploration and mitigating reward hacking.
\end{itemize}

\section{Self-Guided Process Reward Optimization}
\label{sec:spro}

In this section, we propose a novel PRM-free process reinforcement learning framework for token-level Markov Decision Processes (MDPs) in LLMs. 
Our framework
enables the policy model to serve dual roles during optimization: (i) as an \textbf{Actor} module for policy improvement through reinforcement learning, and (ii) as a \textbf{Reward} module for token-level credit assignment during the generation process. 
Therefore, we refer to our framework as Self-Guided Process Reward Optimization.

We introduce our framework through three steps:
\begin{enumerate}
    \item \cref{sec:3.1}: We propose the hypothesis that any LLM can provide credit assignment for token-level MDP, which is independent of specific training objectives used to train LLMs. 
    \item \cref{sec:3.2}: We redefine step-wise advantage by introducing Cumulative Process Reward (CPR) for token-level MDP and
    propose Masked Step Advantage for process RL training.
    \item \cref{sec:3.3}: We introduce the Self-Guided Process Reward Optimization (SPRO) algorithm, detailing its objective function and training procedure.
\end{enumerate}

This self-guided framework has three key advantages: (i) it avoids the reward modeling bottleneck inherent in traditional RLHF pipelines; (ii) credit assignment dynamically improves in alignment with policy improvement, creating a virtuous cycle of mutual refinement; and (iii) it enables effective process reinforcement learning.

\subsection{The process rewards could be self-guided}
\label{sec:3.1}

As shown in~\cref{eq:new_r_function}, the process reward $r(\mathbf{s}_t,\mathbf{a}_t)$ is defined as the log-ratio between the probability of $\ba_t$ under the optimal policy and the given reference policy. \citet{rafailovrfrom} argue that the trained DPO model $\pi^*$ yields the best estimate of an optimal $Q$-function, since the value term $V^*(\mathbf{s})$ is reduced using Bradley-Terry preference model. 
Building upon this,~\citet{yuan2024free} further extend the idea to Cross-Entropy (CE) loss.~\citet{cui2025process} apply CE loss to train an implicit PRM, and subsequently use the resulting process rewards to compute advantages for policy optimization. 

Obviously, the accuracy of $r(\mathbf{s}_t,\mathbf{a}_t)$ directly depends on the quality of the optimal policy. 
Since the policy model is trained to directly approximate $\pi^*$, the near-optimal solution $\pi_{\theta_T}$ inherently provides more accurate rewards than a PRM $\pi_{\varphi}$ trained separately. 
Otherwise, the policy model itself would be inferior, contradicting its optimality hypothesis, which means that the separately trained PRM could be a better solution than our trained policy model $\pi_\theta$.
This observation forms the foundation for our self-guided reward formulation.

\begin{proposition}\label{proposition:1}
    Any LLM is always the optimal soft $Q$-functions for some reward functions in the token-level MDP~\cite{rafailovrfrom}, thus enabling token-level credit assignment. 
    In particular, this equivalence allows the policy model to serve as its own reward model.
\end{proposition}
\begin{proof}
	Let $\bm{\ell}(\ba_t|\bs_t)$ denote the output logits of a given LLM policy $\pi$ for token $\ba_t$ conditioned on state $\bs_t$. 
	We define $Q$-function as a scaled version of the logits: $Q(\bs_t, \ba_t) = \beta \bm{\ell}(\ba_t|\bs_t)$.
	The corresponding partition function is derived by taking the log-sum-exp of logits over all possible actions $\ba \in \mathcal{A}$ and defined as $Z(\bs_t)$. Consequently, the optimal value function corresponding to such $Q$ is exactly $V(\bs_t) = \beta \log Z(\bs_t)$. This yields the following form of the policy: 
	\begin{eqnarray}\label{eq:llm_qfunction}
		\pi(\ba_t|\bs_t) = \softmax(\mathbf{\bm{\ell}(\ba_t|\bs_t)}) = \frac{e^{\mathbf{\bm{\ell}(\ba_t|\bs_t)}}}{\sum_{\mathbf{a} \in \mathcal{A}} e^{\mathbf{\bm{\ell}(\ba |\bs_t)}}} = \frac{e^{Q(\bs_t, \ba_t) / \beta}}{Z(\bs_t)} = e^{\left(Q(\mathbf{s}_t,\mathbf{a}_t) - V(\mathbf{s}_t)\right)/\beta}.
	\end{eqnarray}	
        Eq.~(\ref{eq:llm_qfunction}) shows that any LLM is a soft $Q$-function for some reward function~\cite{rafailovrfrom}. 
         This property is independent of the specific objective used to train the LLM. Hence, the policy model can provide token-level credit assignment for itself.
\end{proof}

This style of relationship is first explored by~\citet{rafailovrfrom} and later in~\citet{zhong2024dpo} and \citet{yuan2024free}. However, these works limit the scope of credit assignment to models trained with specific objectives. In contrast, we generalize this finding to any LLM.

Building on this generalization, our framework leverages
the policy model itself, rather than a separately trained reward model. Since both the policy and reference models are initialized from the same SFT model, the process rewards are initially zero. As training progresses and the policy model shifts away from the reference, the process rewards start contributing to the RL optimization.

\begin{figure}[t]
    \centering
    \includegraphics[width=0.6\linewidth]{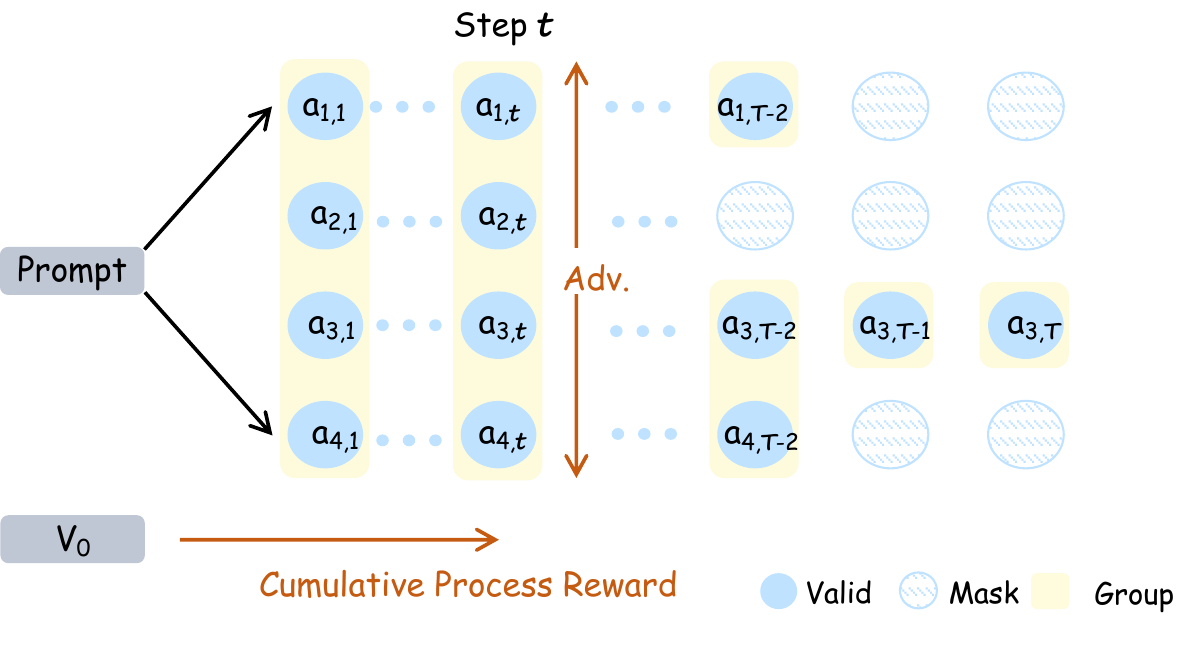}
    \caption{\textbf{Illustration of Masked Step Advantage.} Assume that four responses are sampled for each prompt. At each step $t$, we calculate  cumulative process rewards and further compute the step-level advantages within the vertical valid masked groups, excluding empty step units from all calculations.}
    \label{fig:Masked-Step-Adv}
\end{figure}

\subsection{Redefine Step-wise Advantage}
\label{sec:3.2}
In the previous section, we propose that the process reward can be self-guided by the policy model. In this section, we redefine step-wise advantage by introducing \emph{Cumulative Process Reward (CPR)} and \emph{Masked Step Advantage (MSA)} to facilitate rigorous step-wise action advantage estimation.

\textbf{Cumulative Process Reward (CPR).}\hspace{1em}
We argue that the mechanism of LLMs should inform the design of process rewards in token-level MDPs. 
Since auto-regressive generation employs masked attention, the hidden state at step $t$ inherently encodes all information of the prefix sequence~\cite{vaswani2017attention}, which means each hidden state represents the complete trajectory up to its corresponding time step.
Prior work has effectively utilized this property:~\citet{lightman2023let}  employ the final token's hidden state at each step to predict correctness with PRMs, aligning with extensive probing studies that leverage such representations to analyze model properties~\cite{belinkov2022probing,aspillaga2021inspecting,conneau2018you,dai2022knowledge,geva2021transformer,allen2024physics}. 
Therefore, we propose that intermediate reward signals at step $t$ should similarly capture contributions from all preceding steps, which we formalize as the definition of the Cumulative Process Reward.

Given a policy model $\pi_\theta$ during training iterations~(we omit the iteration subscript for convenience), Proposition~\ref{proposition:1} establishes that there always exists an implicit reward function $r(\bs_t,\ba_t)$ together with corresponding $Q(\bs_t, \ba_t)$ and $V(\bs_t)$ functions. 
Due to the optimality of the $Q$-function introduced in Appendix~\ref{sec:preliminaries}, these functions also satisfy the equality relation in~\cref{eq:new_r_function} as follows:
\begin{eqnarray}\label{eq:iteration_reward_func}
    r(\mathbf{s}_t,\mathbf{a}_t) + V(\mathbf{s}_{t + 1}) - V(\mathbf{s}_t) = \beta\log\frac{\pi_\theta(\mathbf{a}_t|\mathbf{s}_t)}{\pi_{\text{ref}}(\mathbf{a}_t|\mathbf{s}_t)}.
\end{eqnarray}

For an arbitrary time step $t$ within a trajectory $\tau=\{ \bs_0, \ba_0, \ldots, \ba_{T-1}, \bs_{T} \}$, we define a cumulative reward w.r.t. step $t$ by accumulating~\cref{eq:iteration_reward_func} from $0$ to $t$:
\begin{eqnarray}\label{eq:sum_process}
    \sum_{j=0}^{t}\left(r\left(\mathbf{s}_j,\mathbf{a}_j\right) + V\left(\mathbf{s}_{j + 1}\right) - V\left(\mathbf{s}_j\right)\right) = \sum_{j=0}^{t}\beta\log\frac{\pi_\theta(\mathbf{a}_j|\mathbf{s}_j)}{\pi_{\text{ref}}(\mathbf{a}_j|\mathbf{s}_j)}.
\end{eqnarray}

By telescoping the value function $V(\mathbf{s}_{j+1}) - V(\mathbf{s}_j)$ on the left-hand side, we can get the following:
\begin{eqnarray}\label{eq:cumulative_rt_Vt+1}
\sum_{j=0}^{t} r(\mathbf{s}_j, \mathbf{a}_j) + V(\mathbf{s}_{t+1})  = V(\mathbf{s}_0) + \sum_{j=0}^{t} \beta \log \frac{\pi(\mathbf{a}_j | \mathbf{s}_j)}{\pi_{\text{ref}}(\mathbf{a}_j | \mathbf{s}_j)}.
\end{eqnarray}

\citet{nachum2017bridging} exploit similar derivations to enforce temporal consistency by minimizing the squared soft consistency error, a framework~\citet{guo2022efficient} subsequently adopt for text generation. Alternatively,~\citet{xieexploratory} leverage these derivations to analyze trajectory coverability and derive online sampling complexity bounds. Unlike prior works, however, we treat this equation as a surrogate reward signal to facilitate advantage computation for policy optimization.

The left-hand side of~\cref{eq:cumulative_rt_Vt+1} represents cumulative reward up to step $t$ plus future expected return starting from $\mathbf{s}_{t+1}$ (discount factor $\gamma = 1$). Let $\mathcal{R}_t$ denote Cumulative Process Reward (CPR):
\begin{eqnarray}\label{eq:Rt_defination}
    \mathcal{R}_t := \sum_{j=0}^{t} r(\mathbf{s}_j, \mathbf{a}_j) + V(\mathbf{s}_{t+1}) = V(\mathbf{s}_0) + \sum_{j=0}^{t} \beta \log \frac{\pi(\mathbf{a}_j | \mathbf{s}_j)}{\pi_{\text{ref}}(\mathbf{a}_j | \mathbf{s}_j)}.
\end{eqnarray}
We adopt CPR to align with the cumulative nature of LLM representations. Additionally, since all responses start from the same initial state $\mathbf{s}_0$, $\mathcal{R}_{t}$ facilitates subsequent advantage computation.

\textbf{Masked Step Advantage (MSA).}\hspace{1em}
For trajectories $\{\tau_i\}$ of the same prompt, the Cumulative Process Rewards $\{\mathcal{R}_{i,t}\}$ at the same step $t$ are comparable because they all start from the same initial state $\mathbf{s}_0$ and can be regarded as one-step state transition rewards, similar to GRPO that the outcome reward can be seemed as one-step transition rewards.

By introducing $\tilde{\mathcal{R}}_{t}=\sum_{j=0}^{t} \beta \log \frac{\pi_\theta(\mathbf{a}_j | \mathbf{s}_j)}{\pi_{\text{ref}}(\mathbf{a}_j | \mathbf{s}_j)}$, we formally define \textbf{Masked Step Advantage (MSA)} corresponding to the cumulative reward as follows:
\begin{eqnarray}\label{eq:MSA}
\mathrm{MSA}_{i,t} := \mathcal{R}_{i,t} - b_t = \tilde{\mathcal{R}}_{i,t} - \tilde{b_t} = \tilde{\mathcal{R}}_{i,t} - \mathrm{mask\_mean}(\{\tilde{\mathcal{R}}_{i,t}\}),
\end{eqnarray}
where $i$ represents the $i^{\text{th}}$ response.
The term $b_t$ (or $\tilde{b}_t$) denotes the 
group-wise average of valid steps at step $t$, serving as an advantage baseline. Here, a $\mathrm{masked\_mean}$ operator is employed.
Note that the constant $V(\mathbf{s}_0)$ on the right-hand side of \cref{eq:Rt_defination} will cancel out in all responses, making $\mathcal{R}_{i,t}$ and $\tilde{\mathcal{R}}_{i,t}$ equivalent for advantage calculation. 

Taking~\cref{fig:Masked-Step-Adv} as an example, if only the third response contains a valid step at time $T{-}1$, then $b_{T-1} = \mathcal{R}_{3,T-1}$ and $\mathrm{MSA}_{3,T-1} = 0$. This indicates that \textbf{MSA does not introduce the length bias}, since the third response does not gain additional advantage even if it is longer than the others. 
In this way, strict per-step comparisons within shared-prompt sampling groups are achieved without introducing length bias.

\begin{algorithm}[t]
    \caption{Self-Guided Process Reward Optimization (SPRO)}
    \label{alg:spro}
    {\bfseries Input:} Initial policy model $\pi_{\theta_{\text{init}}}$; outcome reward verifier $r_o$; task prompts $\mathcal{D}$.
    \begin{algorithmic}[1] 
        \STATE policy model $\pi_\theta \gets \pi_{\theta_{\text{init}}}$
        \STATE reference model $\pi_{\text{ref}} \gets \pi_{\theta_{\text{init}}}$
        \FOR{iteration $= 1$ {\bfseries to} $K$}
            \STATE Sample a batch $\mathcal{D}_b$ from $\mathcal{D}$
            \STATE Update the old policy model $\pi_{\theta_{\text{old}}} \gets \pi_\theta$
            \STATE Sample $G$ outputs $\{\by_i\}_{i=1}^G \sim \pi_{\theta_{\text{old}}}(\cdot \mid \bx)$ for each $\bx \in \mathcal{D}_b$
            \STATE Compute outcome rewards $\{r_o(\by_i)\}_{i=1}^G$ for each sampled output ${\by}_i$
            \STATE Apply accuracy filter on prompts in $\mathcal{D}_b$
            \STATE Compute the Cumulative Process Reward $\mathcal{R}_{i,t}$ and Masked Step Advantage $\mathrm{MSA}_{i,t}$ for the $t^{\text{th}}$ token of all responses $\{{\by}_i\}$ with \cref{eq:Rt_defination,eq:MSA}
            \STATE Compute $A_{i,t}$ for the $t^{\text{th}}$ token of all responses $\{{\by}_i\}$ through \cref{eq:advantage_function}
            \FOR{iteration $= 1$ {\bfseries to} $\mu$}
                \STATE Update the policy model $\pi_\theta$ by maximizing the SPRO objective \cref{eq:objective_function}
            \ENDFOR
        \ENDFOR 
        
    \end{algorithmic}
    {\bfseries Output:} Policy Model $\pi_\theta$.
\end{algorithm}


\subsection{Self-Guided Process Reward Optimization}
\label{sec:3.3}

Following common practice in Policy Gradient algorithm~\cite{williams1992simple}, we incorporate MSA as a process reward signal into the outcome-supervised RL method Group Relative Policy Optimization (GRPO)~\cite{shao2024deepseekmath}, resulting in the SPRO advantage function:
\begin{eqnarray}\label{eq:advantage_function}
A_{i,t} & = & 
\underbrace{
    \frac{r_{o}\left(\by_{i}\right) - \operatorname{mean}\left(\left\{r_{o}\left(\mathbf{y}_{i}\right)\right\}\right)}
    {\operatorname{std}\left(\left\{r_{o}\left(\mathbf{y}_{i}\right)\right\}\right)}
}_{\text{GRPO with outcome rewards}} + \underbrace{\Big(\mathcal{R}_{i,t} - \operatorname{masked\_mean}\left(\{\mathcal{R}_{i,t}\}\right)\Big)}_{\mathrm{MSA}_{i,t}}.
\end{eqnarray}

Then the policy model can be optimized by maximizing the objective as follows:
\begin{eqnarray}\label{eq:objective_function}
\mathcal{J}_{\text{SPRO}}(\theta) &&= \mathbb{E}_{\bx,\left\{\by_{i}\right\}_{i=1}^{G}\sim\pi_{\theta_{\text{old}}}(\cdot\mid \bx)} \nonumber \frac{1}{G}\sum_{i=1}^{G}\frac{1}{\left|\by_i\right|}  
\sum_{t=1}^{\left|\by_i\right|} 
\min \\
&&\left(
\frac{\pi_{\theta}(y_{i,t}\mid \bx,\by_{i,<t})}{\pi_{\theta_{\text{old}}}(y_{i,t}\mid \bx,\by_{i,<t})}A_{i,t}, 
\operatorname{CLIP}\left(
\frac{\pi_{\theta}(y_{i,t}\mid \bx,\by_{i,<t})}{\pi_{\theta_{\text{old}}}(y_{i,t}\mid \bx,\by_{i,<t})},
1-\varepsilon, 1+\varepsilon
\right)A_{i,t}
\right).
\end{eqnarray}



\cref{alg:spro} illustrates the detailed implementation of our proposed SPRO framework. It can be observed that the calculations of CPR and MSA depend exclusively on the current policy model during training, which motivates our designation of the approach as self-guided. 
Furthermore, our proposed advantage function computes relative advantages by grouping tokens from the identical timestep across all sampled responses, ensuring a less biased advantage estimation.

\section{Experiments}
\label{sec:experiment}
\subsection{Experimental setup}
\vspace{-1pt}
We conducted experiments on mathematical and programming datasets to evaluate our proposed SPRO, focusing on comparing vanilla GRPO~\cite{shao2024deepseekmath} and PRIME~\cite{cui2025process}.
Appendix~\ref{appendix:exp_setup} describes our experimental setup in detail, including \textbf{Base model and Dataset, Resources and Hyper-parameters, Outcome Verifiers, Evaluation Benchmarks and Baseline Algorithms.} 
\vspace{-1pt}

\vspace{-2em}
\subsection{Main Results}
\vspace{-1pt}
Our experiment results demonstrate that SPRO enables effective PRL for reasoning, yielding substantial improvements over GRPO and PRIME.

\textbf{Accuracy Improvement.}\hspace{1em}As shown in~\cref{fig:test_acc}, SPRO achieves a $12.9\%$ higher test accuracy than GRPO~\cite{shao2024deepseekmath}, and $7.2\%$ higher than PRIME~\cite{cui2025process}.  The detailed test accuracy values can be found in Table~\ref{tab:rl_comparison}. The comparative evaluation results in Table~\ref{tab:rl_eval} demonstrates that SPRO significantly enhances the reasoning capabilities of policy models.We further show that this advantage is not limited to the Qwen family: on Llama-3.1-8B-Instruct~\cite{grattafiori2024llama} trained with the same RL dataset, SPRO still outperforms both GRPO and PRIME by 6.7\%, suggesting that our method generalizes across model families (Appendix~\ref{appendix:llama}). Beyond RLVR settings with verifiable rewards, we also observe performance gains in a general-language alignment benchmark, suggesting that the benefit of SPRO extend beyond math and code reasoning (Appendix~\ref{appendix:beyond rlvr}).

\textbf{Token Efficiency.}\hspace{1em} \cref{fig:response_len} shows that SPRO achieves progressively shorter response lengths during training process compared to baseline methods. 
Table~\ref{tab:rl_seq_len} presents the sampling sequence length on the validation set. Our SPRO achieves the best performance among all methods, even with a response length  nearly one-third shorter than vanilla GRPO. 
A case study is presented in Appendix~\ref{appendix:case study}.

\textbf{Training Efficiency.}\hspace{1em} As shown in~\cref{fig:gpu_acc_a}, our SPRO requires only $29\%$ and $15\%$ of the GPU hours needed by vanilla GRPO and PRIME respectively to achieve equivalent accuracy. 
Furthermore, the shorter sampling sequence length generated by SPRO~(as illustrated in~\cref{fig:response_len}) contributes to a significant computational advantage in each optimization step.  This is evidenced by the wall-clock time breakdown in~\cref{tab:wallclock} (Appendix~\ref{appendix:wall clock breakdown}).


\textbf{Entropy Stability.}\hspace{1em} As shown in~\cref{fig:entropy}, unlike PRIME which suffers from entropy collapse, our method maintains effective state-action space exploration during training, preserving optimization efficiency and avoiding reward hacking.

\begin{figure}[t]
    \centering
    \begin{subfigure}[b]{0.4\textwidth} 
        \centering
        \includegraphics[width=\textwidth]{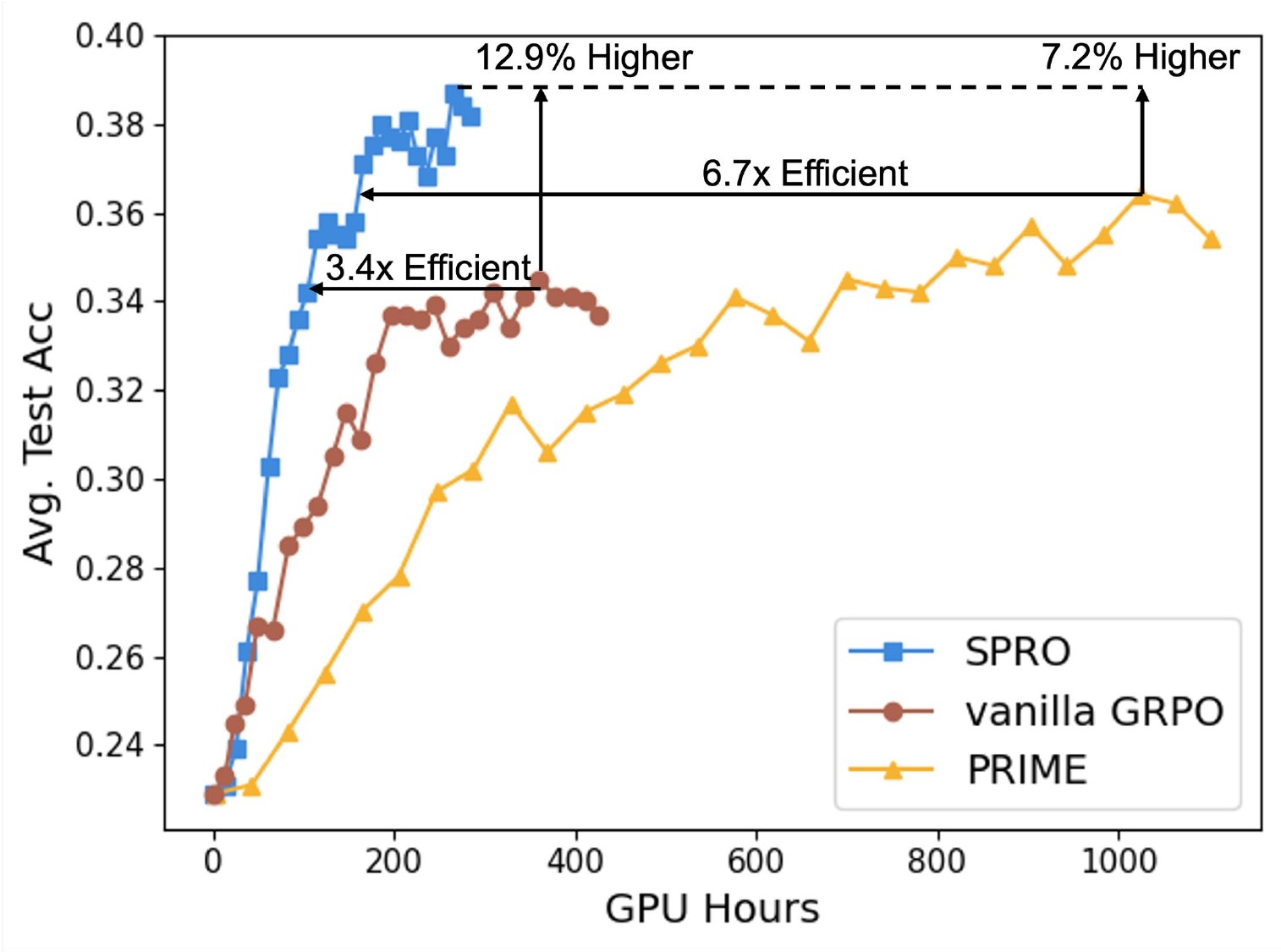}
        \caption{GPU Hours}
        \label{fig:gpu_acc_a}
    \end{subfigure}
    \hspace{0.1\textwidth}
    \begin{subfigure}[b]{0.4\textwidth} 
        \centering
        \includegraphics[width=\textwidth]{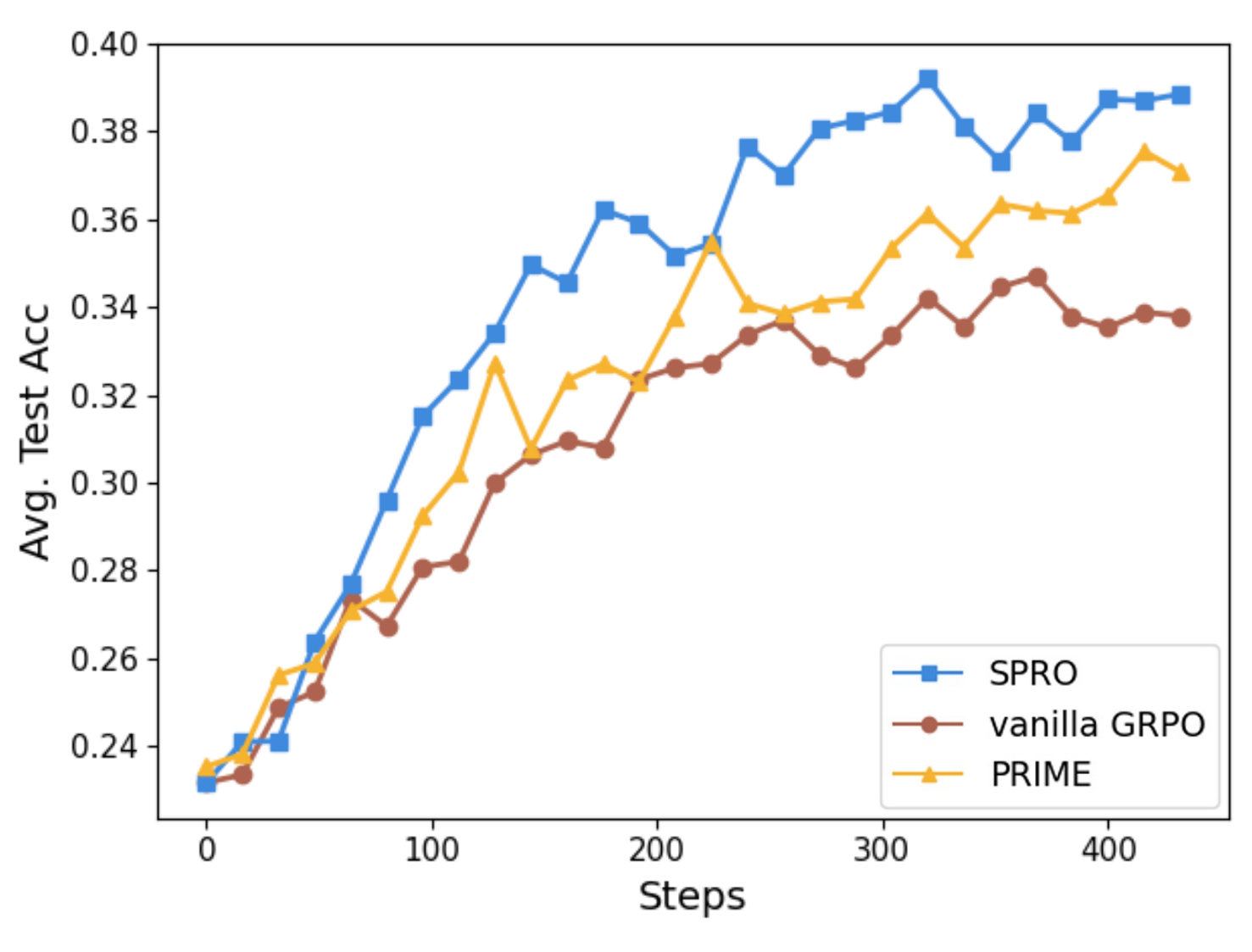}
        \caption{Steps}
        \label{fig:step_acc_b}
    \end{subfigure}
    \caption{Performance of SPRO on math and code benchmarks with respect to GPU hours and training steps. SPRO outperforms the previous SoTA process reinforcement learning method PRIME with 6.7× higher training efficiency and an 7.2\% performance improvement. Notably, SPRO reduces per-step computation time owing to its shorter trajectories.}
    \label{fig:test_acc}
    \vspace{-2em}
\end{figure}


\begin{table*}[t]
  \caption{\textbf{Comparison of evaluation accuracy} between SPRO and other baselines under identical conditions: same base model, training data, and 432 training steps. The baselines are reproduced using the veRL codebase\protect\footnotemark.}
  \label{tab:rl_eval}
  \centering
   \setlength{\tabcolsep}{1.5pt} 
  \begin{tabular*}{0.96\textwidth}{@{\extracolsep{\fill}} l c c c c c c c c}
    \toprule
    Methods   & AIME & AMC  & MATH & Minerva & Olympiad & LeetCode & LiveCode & Avg. \\
    \midrule
    Base Model & 3.33 & 37.35 & 64.00 & 21.69 & 28.15 & 22.22 & 16.72 & 27.64 \\
    vanilla GRPO    & 10.00 & 39.76 & 66.40 & 27.57 & 30.67 & 28.89 & 22.70 & 32.28 \\
    PRIME & 10.00 & 43.37 & 72.20 & 28.68 & 29.04 & 28.33 & 22.16 & 33.40 \\
    \textbf{Ours (SPRO)} & \textbf{13.33} & \textbf{45.78} & \textbf{74.20} & \textbf{30.15} & \textbf{39.11} & \textbf{29.44} & \textbf{24.00} & \textbf{36.57}\\
    \bottomrule
  \end{tabular*}
\end{table*}

\begin{figure}[t]
    \centering
    \begin{subfigure}[b]{0.4\textwidth} 
        \centering
        \includegraphics[width=\textwidth]{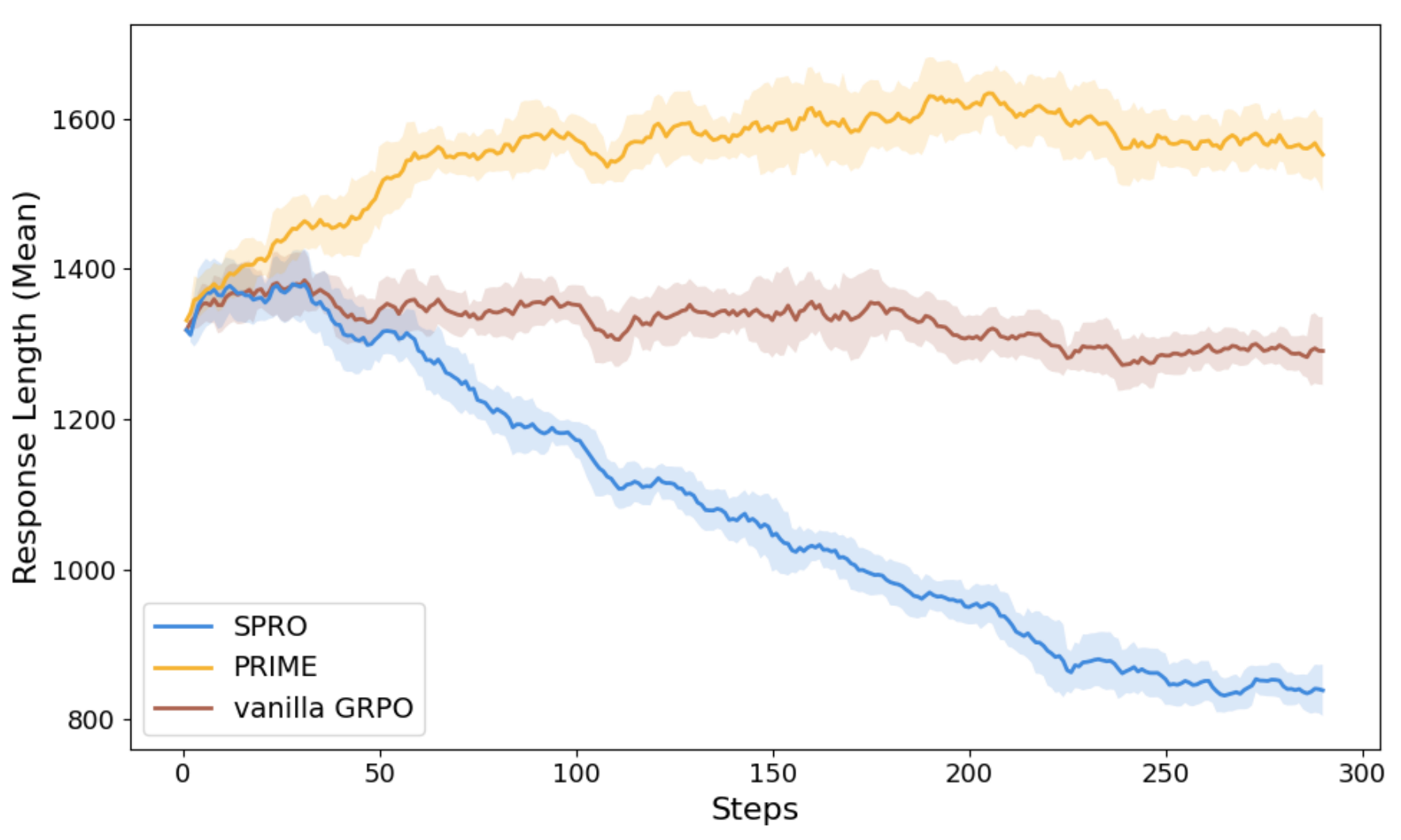}
        \caption{Average length of policy model-generated responses.}
        \label{fig:response_len}
    \end{subfigure}
    \hspace{0.1\textwidth}
    \begin{subfigure}[b]{0.4\textwidth} 
        \centering
        \includegraphics[width=\textwidth]{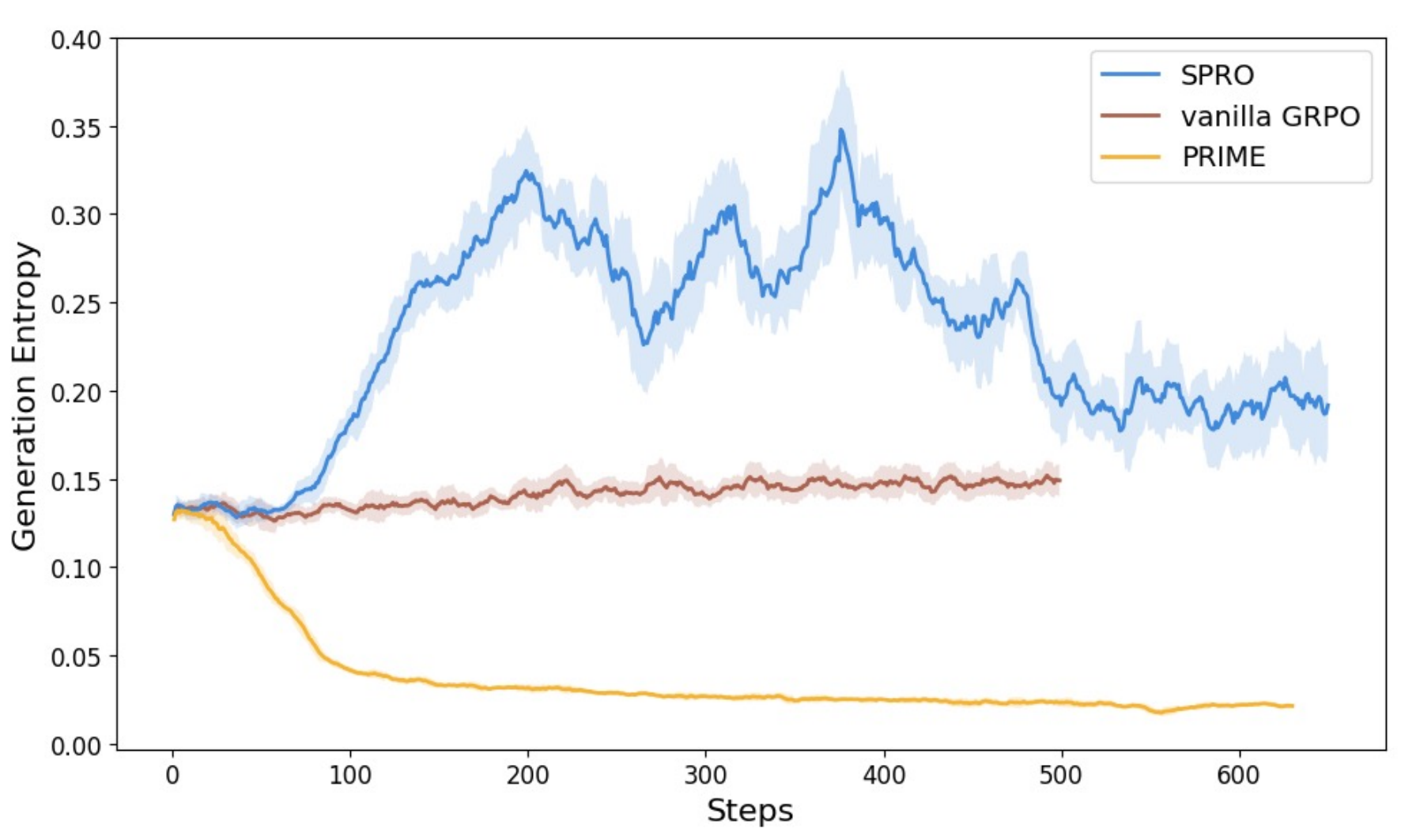}
        \caption{Entropy of policy model’s generation probabilities.}
        \label{fig:entropy}
    \end{subfigure}
    \caption{Effective PRL enables efficient reasoning traces and exploration of action space.}
    \label{fig:response_entropy}
\end{figure}


\subsection{Analysis}
\label{sec:analysis}
\subsubsection{SPRO enables efficient reasoning trajectories}

\cref{fig:response_len} shows that SPRO reduces the average response length of vanilla GRPO~\cite{shao2024deepseekmath} by nearly one-third while improving test accuracy by $12.9\%$, demonstrating the effectiveness of our framework in process reinforcement learning.

\footnotetext{https://github.com/volcengine/verl}


This improvement is primarily attributed to the rigorous step-wise comparison mechanism introduced in~\cref{eq:MSA}. Our approach provides the policy model with MSA feedback at each generation step, 
enabling it to identify which tokens contribute positively to the overall return. Such a fine-grained feedback mechanism effectively encourages more concise and task-focused output. In contrast, existing methods such as PRIME~\cite{cui2025process} employ a coarser reward signal by averaging the returns across all trajectories and timesteps. 
This design results in advantages
that are relative to both groups and timesteps, consequently diminishing the effectiveness of the intended group-level comparison.

The Cumulative Process Rewards defined in~\cref{eq:Rt_defination} not only provide the 
foundation for MSA in~\cref{eq:MSA}, but also exhibit intrinsic alignment with the hidden state dynamics of LLMs. 
This insight, inherent in the mechanism of LLMs  and different from conventional RL scenarios, merits further attention in future research on reinforcement learning for LLMs.

\subsubsection{SPRO enables exploration of action space}

The collapse of policy entropy is a widely observed phenomenon in reinforcement learning, as documented in prior works~\cite{cui2025entropy, yu2025dapo}, and our experimental results are consistent with this trend. 
In our experiments, the entropy coefficient is fixed at $0.001$ across methods. 
This persistent exploration constitutes a key feature of SPRO, enabling longer and more effective RL training. By maintaining output diversity, our SPRO prevents premature convergence to suboptimal behaviors while preserving the potential for further improvement. Cui et al. demonstrate that the policy entropy naturally decreases when high-advantage actions already have high probability, but increases when the model selects rare yet high-advantage actions~\yrcite{cui2025entropy}. 
SPRO encourages the latter behavior, validating the effectiveness of our advantage function design in~\cref{eq:advantage_function}. This exploration mechanism directly contributes to the $12.9\%$ improvement in test accuracy over vanilla GRPO. Importantly, this performance gain stems not from implementation tricks such as policy loss clipping, but from genuine exploration dynamics.

In particular, our SPRO successfully combines active exploration with more concise reasoning trajectories. 
The reduced response length does not mean shortcutting; instead, the policy model thoroughly explores the state-action space while strategically selecting concise, diverse, and effective solutions.
This demonstrates a form of intelligent exploration, where the model identifies efficient solutions without compromising the correctness or diversity.

\subsubsection{SPRO enables industry-scale process reinforcement learning}

\begin{table}[t]
  \caption{\textbf{Scaling experiment results} comparing performance across different training steps. SPRO scales to larger models and continues to provide stable improvements.}
  \label{tab:scaling_results}
  \centering
  
  \begin{tabular*}{\linewidth}{@{\extracolsep{\fill}} l c c c c c}
    \toprule
    Step & 0 & 32 & 64 & 128 & 200 \\
    \midrule
    Qwen2.5-14B-Instruct & 0.382 & 0.422 & 0.442 & 0.439 & 0.458 \\
    Eurus-2-7B-SFT       & 0.227 & 0.239 & 0.277 & 0.336 & 0.358 \\
    \bottomrule
  \end{tabular*}
\end{table}

Even in advanced industrial Large Reasoning Models, PRMs have been identified as a failure case due to their inherent limitations~\cite{guo2025deepseek}. 
Training inefficiency significantly reduces its potential benefits. This is clearly demonstrated in~\cref{fig:test_acc} when the x-axis is changed from training steps to GPU hours. However, SPRO effectively addresses the training inefficiency, keeping the dual-model framework comparable to outcome-supervised methods.

As shown in~\cref{fig:gpu_acc_a}, our method reduces the additional computational overhead. Furthermore,~\cref{tab:scaling_results} demonstrates that SPRO scales effectively to larger models, yielding consistent performance gains. While hardware constraints limited our evaluation to 7B and 14B models, SPRO shows considerable potential for industrial-scale implementation.

\subsection{Ablation Study}
\label{sec:ablation study}
\begin{table}[t]
  \caption{\textbf{Ablation Study for SPRO.} Evaluation accuracy under different component settings.}
  \label{tab:ablation_results}
  \centering
  \begin{tabular*}{\linewidth}{@{\extracolsep{\fill}} l c c c c c}
    \toprule
    Step & 120 & 160 & 400 & 412 & 596 \\
    \midrule
    Base & 0.356 & \textbf{0.114} & -- & -- & -- \\
    + CPR  & 0.338 & 0.358 & 0.384 & \textbf{0.111} & -- \\
    +MSA (SPRO)   & 0.338 & 0.354 & 0.387 & 0.385 & 0.399 \\
    \bottomrule
  \end{tabular*}
\end{table}

We analyze the contributions of CPR and MSA in our methodology, as detailed in~\cref{tab:ablation_results}. The results show that both components are essential for the performance gains and training stability
in SPRO. For the base method, in the absence of CPR and MSA, we adopt $r= \beta \log \frac{\pi_\theta}{\pi_{\text{ref}}}$ as the reward signal following previous work~\cite{rafailov2023direct,rafailovrfrom} and the grouping strategy in PRIME~\cite{cui2025process} . As shown in~\cref{tab:ablation_results}, without the full SPRO framework, the model suffers from sudden performance degradation (e.g., accuracy dropping to around 0.11), accompanied by erratic entropy and response-length trajectories.

\section{Conclusion}
\label{sec:conclusion}

In this work, we introduced Self-Guided Process Reward Optimization~(SPRO), a novel and scalable RL framework for LLMs that eliminates the dependency on costly PRMs while preserving the simplicity of outcome-supervised RL. By introducing Cumulative Process Reward~(CPR) as a surrogate for self-guided process signals and proposing Masked Step Advantage~(MSA), our method enables rigorous step-level advantage estimation through shared-prompt comparisons.
Experiments demonstrate that SPRO significantly improves both accuracy and training efficiency. In particular, it addresses two critical challenges in RL for LLMs: 
token efficiency
and 
policy entropy collapse.
The scalability and ease of deployment of SPRO make it particularly suitable for industrial implementation, offering a practical and effective alternative to traditional process reward approaches. 



\bibliographystyle{unsrtnat}
\bibliography{uploads/ref}

\newpage
\appendix
\label{sec:appendix}




\section{Experimental Setup}
\label{appendix:exp_setup}
\textbf{Base model and Dataset.}\hspace{1em}We adopt Eurus-2-7B-SFT~\cite{cui2025process} as our base model, which is fine-tuned from Qwen2.5-Math-7B-Base~\cite{qwen2025qwen25technicalreport} on mathematical and programming tasks.
The RL dataset is Eurus-2-RL-Data~\cite{cui2025process}, which contains math problems ranging from the high school level to International Mathematical Olympiad competition questions, as well as programming tasks primarily at the competitive programming level. 

\textbf{Resources and Hyper-parameters.}\hspace{1em}All experiments are conducted on a single node equipped with 8x NVIDIA A800 GPUs (80G memory) using the veRL framework~\cite{sheng2025hybridflow}.
For optimization, we use the AdamW optimizer with a cosine decay learning rate schedule, initialized at $1 \times 10^{-6}$. 
For rollout stage, we collect 256 prompts with an oversampling factor of 2 and generate 4 responses per prompt. We apply an accuracy filtering threshold between 0.2 and 0.8, and prompts that fall within this range are prioritized. 
For training, the batch size is 256 and the micro batch size is 16. The KL coefficient is set to 0 which disables the kl-in-loss in veRL but enables kl-in-reward aligned with~\cref{eq:Rt_defination} and the entropy coefficient is set to 0.001 for all experiments.

\textbf{Outcome Verifiers.}\hspace{1em}We employ rule-based reward modeling for outcome verification, aligning with common practices in the recent literature~\cite{guo2025deepseek,cui2025process,lambert2024t}. 
Specifically, we assign a score of 1 if a math answer matches the ground truth and 0 otherwise. For code generation task, we compute the reward as the proportion of passed test cases.


\textbf{Evaluation Benchmarks.}\hspace{1em}We evaluate models on AIME 2024~\cite{li2024numinamath}, AMC~\cite{li2024numinamath}, MATH-500~\cite{hendrycks2021measuringmathematicalproblemsolving}, Minerva Math~\cite{lewkowycz2022solving}, OlympiadBench~\cite{he2024olympiadbench}, LeetCode~\cite{guo2024deepseek}, and LiveCodeBench(v2)~\cite{jainlivecodebench}.

\textbf{Baseline Algorithms.}\hspace{1em}We use vanilla GRPO~\cite{shao2024deepseekmath} as the baseline algorithm for outcome-supervised RL training, and the previous SoTA PRIME~\cite{cui2025process} for process-supervised RL training. We maintain consistency in the base model, training dataset, and key hyperparameters with those employed in our proposed method.

\section{Additional Experimental Results}

\subsection{Test Accuracy}

\label{appendix:test acc table}

\cref{tab:rl_comparison} shows the test accuracy of SPRO and other baselines during training.
\begin{table*} 
  \caption{\textbf{Comparison of test accuracy} between SPRO and other baselines using the same base model and training data. We present the highest pass@1 test accuracy scores of four math benchmarks and two code benchmarks.}
  \label{tab:rl_comparison}
  \centering
  \begin{tabular*}{0.96\textwidth}{@{\extracolsep{\fill}} l c c c c c c c c}
    \toprule
    Methods   & AMC & MATH  & Olympiad & K12 & CodeForces & CodeContests & Avg. \\
    \midrule
    Base Model & 21.1 & 48.1 & 13.9 & 40.0 & 5.2 & 11.3 & 22.8 \\
    vanilla GRPO    & 24.3 & 52.0 & 21.3 & 49.5 & 28.9 & 29.2 & 34.2 \\
    PRIME & 31.2 & 52.7 & 25.4 & 54.6 & 26.4 & 25.9 & 36.0 \\
    \textbf{Ours (SPRO)} & \textbf{31.9} & \textbf{53.6} & \textbf{28.2} & \textbf{55.7} & \textbf{29.6} & \textbf{32.4} & \textbf{38.6} \\
    \bottomrule
  \end{tabular*}
\end{table*}

\subsection{Response Length}

\cref{tab:rl_seq_len} shows the response length of SPRO and other baselines during evaluation.
\begin{table}
  \caption{\textbf{Comparison of response length} between SPRO and other baselines under identical conditions: same base model, training data, and 432 training steps.}
  \label{tab:rl_seq_len}
  \centering
  \setlength{\tabcolsep}{1.5pt} 
  
  \begin{tabular*}{\linewidth}{@{\extracolsep{\fill}} l c c c c c}
    \toprule
    Methods   & AIME & AMC  & MATH & Minerva & Olympiad  \\
    \midrule
    Base Model & 1893 & 1571 & 951 & 762 & 1206 \\
    vanilla GRPO    & 2028 & 1635 & 927 & 816 & 1263 \\
    PRIME & 1731 & 1496 & 920 & 983 & 1323  \\
    \textbf{Ours (SPRO)} & \textbf{1349} & \textbf{1024} & \textbf{685} & \textbf{701} & \textbf{947}  \\
    \bottomrule
  \end{tabular*}
\end{table}

\subsection{Per-step Wall-clock Breakdown}
\label{appendix:wall clock breakdown}
For a more transparent efficiency comparison, we further provide a per-step wall-clock breakdown of SPRO. Each training step is decomposed into three stages, including rollout sampling, advantage and reward computation, and parameter update, and the corresponding time cost of each stage is reported in Table~\ref{tab:wallclock}.

\begin{table}
  \caption{\textbf{Wall-clock time breakdown per training step.} We report the average wall-clock time per step and its decomposition into sampling, advantage and reward-model computation, and parameter update.}
  \label{tab:wallclock}
  \centering
  \begin{tabular}{lcccc}
    \toprule
    Method & Wall-clock/step (s) & Sampling (s) & Adv/RM (s) & Update (s) \\
    \midrule
    PRIME & 1068.0 & 210.8 & 475.1 & 381.9 \\
    GRPO & 272.3 & 195.7 & 0.07 & 76.4 \\
    SPRO (Ours) & 228.0 & 169.3 & 0.20 & 58.3 \\
    \bottomrule
  \end{tabular}
\end{table}

\subsection{Cross-family Generalization}
\label{appendix:llama}
We further evaluate SPRO on Llama-3.1-8B-Instruct using Eurus-2-RL-Data, the same RL dataset used for the Qwen-family experiments. After training for 224 steps, we obtain the results shown in Table~\ref{tab:llama_generalization}.

Note that the base model Llama-3.1-8B-Instruct is not a fine-tuned model for math and programming tasks, which makes Eurus-2-RL-Data more challenging in this setting. Nevertheless, SPRO still outperforms both GRPO and PRIME, improving over the baselines by 6.7\%. This result suggests that the benefit of SPRO is not specific to the Qwen family, but generalizes to other model families.

\begin{table}[t]
  \caption{\textbf{Generalization results on Llama-3.1-8B-Instruct.} We report the average test performance after 224 RL training steps on Eurus-2-RL-Data.}
  \label{tab:llama_generalization}
  \centering
  \begin{tabular}{lcccc}
    \toprule
    Methods & Base Model & vanilla GRPO & PRIME & Ours (SPRO) \\
    \midrule
    Avg. Test & 20.4 & 21.0 & 20.9 & 22.4 \\
    \bottomrule
  \end{tabular}
\end{table}

\subsection{Beyond RLVR}
\label{appendix:beyond rlvr}
To examine whether the advantage of SPRO extends beyond math and code reasoning, we evaluate it on a general-language alignment benchmark, full-hh-rlhf, a public helpful-and-harmless preference dataset~\cite{bai2022training}. For this experiment, we train both GRPO and SPRO from Llama-3.2-3B-Instruct for 128 steps. During training, we use Skywork-Reward-Llama-3.1-8B~\cite{liu2024skywork} as the reward model.

For evaluation, we follow a widely used pairwise LLM-as-a-judge protocol based on~\cite{zheng2023judging}, using a pairwise comparison prompt adapted to the helpful or harmless dialogue setting. We sample 300 prompts from the held-out validation split and use GPT-4o as the judge. To reduce position bias, we randomly swap the presentation order of the two candidate responses for each prompt. The results are shown in Table~\ref{tab:alignment_generalization}.

These results indicate that SPRO also improves performance in a general-language alignment setting, beyond math and code reasoning.

\begin{table}
  \caption{\textbf{Pairwise evaluation on the full-hh-rlhf benchmark.} We report the win rates of SPRO against GRPO under GPT-4o judgment on 300 held-out prompts.}
  \label{tab:alignment_generalization}
  \centering
  \begin{tabular}{lccc}
    \toprule
    Compared Methods & Win (\%) & Tie (\%) & Lose (\%) \\
    \midrule
    SPRO vs. GRPO & 64.33 & 18.67 & 17.00 \\
    \bottomrule
  \end{tabular}
\end{table}

\section{Preliminaries}
\label{sec:preliminaries}
In this section, we first introduce the token-level MDP for large language models, along with some definitions of reward and objective function in reinforcement learning.
\subsection{Token-level MDP for Large Language Models}

Following with~\citet{rafailovrfrom},~\citet{yuan2024free} and~\citet{cui2025process}, the token-level MDP is defined as a tuple $\mathcal{M} = \left( \mathcal{S}, \mathcal{A}, f, r, \rho \right)$.
\begin{enumerate}[leftmargin=1.0em, topsep=0pt, itemsep=0pt, label=\textbullet]
	\item The action space $\mathcal{A}$ consists of the token vocabulary of any given large language model, while the state space $\mathcal{S}$ comprises all sub-sequences during the inference process. For example, state $\bs_t$ at timestep $t$ could be represented as $\bs_t=\left(\mathbf{x}, \mathbf{y}_{< t}\right)$, where $\mathbf{x}$ is the initial input~(prompt) and $\mathbf{y}_{< t}$ is the sequence of tokens generated up to step $t-1$. 
	\item $f(\bs_t, \ba_t)$ represents a state transition model that updates the state $\bs_{t+1}$ by concatenating the newly generated token $\ba_t$ to $\bs_{t}$. Formally, this can be expressed as $\bs_{t+1} = f(\bs_{t}, a_{t})$. 
	\item $\rho(\bs_t)$ represents a state distribution constraint that limits the sampling range for each state $\bs_t$.
	\item $r(\bs_t, \ba_t)$ denotes the token-level reward given after the model outputs token $\ba_t$ with input state $\bs_t$.
\end{enumerate}

\subsection{Maximum Entropy Reinforcement Learning in the Token-level MDP}
\label{sec:2_MaxEnt}

Given a well-defined token-level MDP $\mathcal{M} = \left( \mathcal{S}, \mathcal{A}, f, r, \rho \right)$, we can optimize the process reinforcement policy $\pi_\theta$ by using the following entropy-augmented~\cite{williams1991function, ziebart2010modeling,schulman2017equivalence}, KL-constrained objective~\cite{ouyang2022training,rafailovrfrom}:
\begin{eqnarray}\label{eq:MaxEnt}
	\max_{\pi_\theta} \, \mathbb{E}_{
		\substack{
			\bs_0 \sim \rho(\bs_0), \\
			\ba_t \sim \pi_\theta(\cdot|\bs_t)
		}
	} \left[ 
	\sum_{t=0}^T \left(( 
	\underbrace{r(\bs_t, \ba_t)}_{\text{token reward}} 
	+ \beta \underbrace{\log{\pi_{\mathrm{ref}}(\ba_t|\bs_t)}}_{\text{KL penalty}} 
	\right) + \beta \underbrace{\mathcal{H}(\pi_\theta)}_{\text{entropy}} 
	\right].
\end{eqnarray}
As mentioned in~\citet{rafailovrfrom} and~\citet{ziebart2010modeling}, in
the general maximum entropy RL setting, the fixed point solution of~\cref{eq:MaxEnt} is given as:
\begin{eqnarray}\label{eq:policy_function}
	\pi^*(\mathbf{a}_t|\mathbf{s}_t) = e^{\left(Q^*(\mathbf{s}_t,\mathbf{a}_t) - V^*(\mathbf{s}_t)\right)/\beta},
\end{eqnarray}
where $\pi^*(\mathbf{a}|\mathbf{s})$ is the optimal policy and $Q^*(\mathbf{s},\mathbf{a})$ is the corresponding optimal soft $Q$-function. The optimal value function $V^*$ is defined as:
\begin{eqnarray}\label{eq:value_function}
	V^*(\mathbf{s}_t) = \beta \log \sum_{\mathbf{a} \in \mathcal{A}} e^{Q^*(\mathbf{s}_t,\mathbf{a})/\beta}.
\end{eqnarray}
As shown in ~\cref{eq:policy_function}, the relationship between the reward function $r(\bs, \ba)$ and optimal policy $\pi(\ba|\bs)$ is not a direct mapping. Instead, the policy is expressed through $Q$-function and V-function, which themselves represent estimates of total future returns. 

To further investigate the reward-policy relationship, Rafailov et al. introduced a modified 
bellman equation between reward function and value functions using KL-divergence penalty~\yrcite{rafailovrfrom}, where:
\begin{eqnarray}
	Q^*(\bs_t, \ba_t) = 
	\begin{cases}
		r(\bs_t, \ba_t) + \beta \log \pi_{\text{ref}}(\ba_t|\bs_t) + V^*(\bs_{t+1}), & \text{if } \bs_{t+1} \text{ is not terminal}, \\
		r(\bs_t, \ba_t) + \beta \log \pi_{\text{ref}}(\ba_t|\bs_t), & \text{if } \bs_{t+1} \text{ is terminal}.
	\end{cases}
	\label{eq:q_function}
\end{eqnarray}
It should be noted that some prior work~\cite{garg2021iq,hejna2024contrastive} has also proposed similar definitions, but they required an assumption that the discount factor $\gamma<1$. Rafailov et al. further proved that the relationship in~\cref{eq:q_function} is indeed one-to-one in the token MDP under mild assumptions~\yrcite{rafailovrfrom}, which means there is a \textbf{bijection} between reward functions and corresponding optimal $Q$-functions in the token-level MDP.

By log-linearizing the optimal policy fixed point in Eq.~(\ref{eq:policy_function}) 
\begin{eqnarray}\label{eq:log-linearize_policy}
	\beta \log \pi^*(\mathbf{a}_t | \mathbf{s}_t) = Q^*(\mathbf{s}_t, \mathbf{a}_t) - V^*(\mathbf{s}_t).
\end{eqnarray}
Substituting in the Bellman equation from~\cref{eq:q_function}~\cite{nachum2017bridging,watson2023coherent}, we have the following function:
\begin{eqnarray}
	r(\mathbf{s}_t,\mathbf{a}_t) + V^*(\mathbf{s}_{t + 1}) - V^*(\mathbf{s}_t) = \beta\log\frac{\pi^*(\mathbf{a}_t|\mathbf{s}_t)}{\pi_{\text{ref}}(\mathbf{a}_t|\mathbf{s}_t)}.
	\label{eq:new_r_function}
\end{eqnarray}
This establishes a mathematical relationship between the process reward function and the policy model.

\section{Related Work}
\label{sec:related_work}
\textbf{From RLHF to RLVR.}\hspace{1em}
In the LLM domain, RL is initially employed to align the model outputs with human preference, a paradigm known as Reinforcement Learning with Human Feedback (RLHF)~\cite{ouyang2022training, bai2022constitutional}. As the focus shifts from alignment to enhancing model reasoning abilities in domains such as math and code, Reinforcement Learning with Verifiable Rewards (RLVR) emerges as a new trend. In this setting, outcome correctness can be verified against ground truth. During the transition, REINFORCE-style algorithms~\cite{shao2024deepseekmath,ahmadian2024rloo,liu2025understanding,williams1992simple,yu2025dapo,hu2025reinforce++}, which rely on Monte Carlo estimator of trajectory returns instead of value functions as used in PPO~\cite{schulman2017proximal}, have demonstrated superior efficiency and performance. However, most existing RLVR methods remain outcome-supervised, limiting their ability to provide well-defined process rewards during training.

\textbf{Explicit Rewards.}\hspace{1em}
\citet{uesato2022solving} and \citet{lightman2023let} show that PRMs outperform ORMs on reasoning tasks. Building on this, subsequent works~\cite{li2023making,yu2023outcome,zhu2023solving,ma2023let,jiao2024learning,havrilla2024glore} train PRMs to either verify reasoning steps or predict the final return, enabling improved trajectory collection through reranking or filtering. Other works~\cite{feng2023alphazero,zhang2024rest,chen2024step} train PRMs for tree-search.

\textbf{Implicit Rewards.}\hspace{1em}Implicit rewards are broadly adopted in LLM alignment~\cite{rafailovrfrom,ethayarajh2024kto,zhouweak,chen2024bootstrapping}.~\citet{rafailovrfrom} show that DPO implicitly learns a Q-function.~\citet{yuan2024free} extend this idea to CE objective.~\citet{zhouweak} use dense implicit rewards for beam search.~\citet{cui2025process} and~\citet{zhong2024dpo} train standalone models to generate implicit rewards for RL training.  



\section{Case Study}
\label{appendix:case study}
\begin{longtable}{@{}p{\dimexpr\textwidth-2\tabcolsep\relax}@{}}
\toprule
\textbf{Question:} \\
Simplify $\frac{1+\sqrt{2}}{2+\sqrt{3}}$. Your solution can be converted to the form $A(1+\sqrt{B})-(\sqrt{C}+\sqrt{D})$, where $A$, $B$, $C$, and $D$ are positive integers. What is $A+B+C+D$? \\
\midrule
\textbf{Ours (SPRO):}  \\
 The given expression is $\frac{1+\sqrt{2}}{2+\sqrt{3}}$. To simplify this, we can rationalize the denominator by multiplying both the numerator and the denominator by the conjugate of the denominator, which is $2-\sqrt{3}$.\\  Multiply the numerator and the denominator by the conjugate of the denominator: $\frac{1+\sqrt{2}}{2+\sqrt{3}} \cdot \frac{2-\sqrt{3}}{2-\sqrt{3}} = \frac{(1+\sqrt{2})(2-\sqrt{3})}{(2+\sqrt{3})(2-\sqrt{3})}$.\\ Expand the numerator: $(1+\sqrt{2})(2-\sqrt{3}) = 2 - \sqrt{3} + 2\sqrt{2} - \sqrt{2}\sqrt{3} = 2 - \sqrt{3} + 2\sqrt{2} - \sqrt{6}$.\\ Expand the denominator: $(2+\sqrt{3})(2-\sqrt{3}) = 2^2 - (\sqrt{3})^2 = 4 - 3 = 1$.\\ So, the expression simplifies to: $\frac{2 - \sqrt{3} + 2\sqrt{2} - \sqrt{6}}{1} = 2 - \sqrt{3} + 2\sqrt{2} - \sqrt{6}$.\\ Rearrange the terms to match the form $A(1+\sqrt{B})-(\sqrt{C}+\sqrt{D})$: $2 + 2\sqrt{2} - \sqrt{3} - \sqrt{6} = 2(1 + \sqrt{2}) - (\sqrt{3} + \sqrt{6})$.\\ Now we have the expression in the desired form, where $A = 2$, $B = 2$, $C = 3$, and $D = 6$.\\  The final expression is $2(1+\sqrt{2})-(\sqrt{3}+\sqrt{6})$, and the sum of the constants is $A+B+C+D = 2+2+3+6 = 13$. Therefore, the answer is 13.    \\
\midrule
\textbf{GRPO:}  \\
The problem asks us to simplify the given expression $\frac{1+\sqrt{2}}{2+\sqrt{3}}$ and then convert it into a specific form. We need to identify the key elements, constraints, and objectives. The key elements are the given expression and the target form. The objective is to simplify the expression and then convert it into the required form. To simplify the given expression, we can start by rationalizing the denominator, which means eliminating the radical in the denominator. We can achieve this by multiplying both the numerator and the denominator by the conjugate of the denominator, which is $2-\sqrt{3}$. This will help us to get rid of the radical in the denominator and simplify the expression. The expression becomes $\frac{(1+\sqrt{2})(2-\sqrt{3})}{(2+\sqrt{3})(2-\sqrt{3})}$.\\ Now, we can simplify the numerator and the denominator separately. Now, let's simplify the numerator and the denominator separately. The numerator becomes $(1+\sqrt{2})(2-\sqrt{3}) = 2 - \sqrt{3} + 2\sqrt{2} - \sqrt{6}$. The denominator becomes $(2+\sqrt{3})(2-\sqrt{3}) = 4 - 3 = 1$. So, the simplified expression is $2 - \sqrt{3} + 2\sqrt{2} - \sqrt{6}$.\\ We can rewrite this expression in the form $A(1+\sqrt{B})-(\sqrt{C}+\sqrt{D})$ by rearranging the terms. Let's try to group the terms with the same radical together. We can rewrite the expression $2 - \sqrt{3} + 2\sqrt{2} - \sqrt{6}$ as $2 + 2\sqrt{2} - \sqrt{3} - \sqrt{6}$.\\ Now, we can factor out a common term from the first two terms: $2(1 + \sqrt{2}) - \sqrt{3} - \sqrt{6}$.\\ We can see that the expression is now in the form $A(1+\sqrt{B})-(\sqrt{C}+\sqrt{D})$, where $A=2$, $B=2$, $C=3$, and $D=6$.\\ Therefore, we have successfully converted the expression into the required form. Now, we can calculate the sum $A+B+C+D$ to get the final answer. The given expression $\frac{1+\sqrt{2}}{2+\sqrt{3}}$ can be simplified by rationalizing the denominator and then rearranging the terms to match the required form. After simplifying, we get $2(1 + \sqrt{2}) - \sqrt{3} - \sqrt{6}$, which is in the form $A(1+\sqrt{B})-(\sqrt{C}+\sqrt{D})$. Therefore, $A=2$, $B=2$, $C=3$, and $D=6$. The sum $A+B+C+D$ is $2+2+3+6 = 13$. \\
13  \\
\midrule
\textbf{PRIME:}  \\
The problem asks us to simplify the given expression $\frac{1+\sqrt{2}}{2+\sqrt{3}}$ and then convert it to a specific form. We need to identify the key elements, constraints, and objectives. The key elements are the given expression and the target form. The constraint is that the solution must be in the specified form. The objective is to simplify the expression and find the values of $A$, $B$, $C$, and $D$. To simplify the given expression $\frac{1+\sqrt{2}}{2+\sqrt{3}}$, we can start by rationalizing the denominator. This involves multiplying the numerator and denominator by the conjugate of the denominator to eliminate the radical in the denominator. The conjugate of $2+\sqrt{3}$ is $2-\sqrt{3}$. We will multiply the expression by $\frac{2-\sqrt{3}}{2-\sqrt{3}}$, which is equivalent to 1.\\ The expression can be broken down into the following steps:1. Multiply the numerator and denominator by the conjugate of the denominator: $\frac{1+\sqrt{2}}{2+\sqrt{3}} \cdot \frac{2-\sqrt{3}}{2-\sqrt{3}}$ 2. Simplify the numerator and denominator separately. 3. Combine the simplified numerator and denominator to obtain the simplified expression\\We will now perform the multiplication and simplification steps outlined in the previous action.\\1. Multiply the numerator and denominator by the conjugate of the denominator:$\frac{1+\sqrt{2}}{2+\sqrt{3}} \cdot \frac{2-\sqrt{3}}{2-\sqrt{3}} = \frac{(1+\sqrt{2})(2-\sqrt{3})}{(2+\sqrt{3})(2-\sqrt{3})}$\\2. Simplify the numerator:$(1+\sqrt{2})(2-\sqrt{3}) = 2 - \sqrt{3} + 2\sqrt{2} - \sqrt{6}$\\3. Simplify the denominator:$(2+\sqrt{3})(2-\sqrt{3}) = 2^2 - (\sqrt{3})^2 = 4 - 3 = 1$\\4. Combine the simplified numerator and denominator:$\frac{2 - \sqrt{3} + 2\sqrt{2} - \sqrt{6}}{1} = 2 - \sqrt{3} + 2\sqrt{2} - \sqrt{6}$\\We have obtained the simplified expression $2 - \sqrt{3} + 2\sqrt{2} - \sqrt{6}$. Now, we need to convert this expression to the target form $A(1+\sqrt{B})-(\sqrt{C}+\sqrt{D})$.\\We can start by rearranging the terms to group the constants and the square root terms:$2 - \sqrt{3} + 2\sqrt{2} - \sqrt{6} = (2 + 2\sqrt{2}) - (\sqrt{3} + \sqrt{6})$\\We can factor out a common factor from the first group of terms:$2 + 2\sqrt{2} = 2(1 + \sqrt{2})$\\Now, we can rewrite the expression in the target form:$2(1 + \sqrt{2}) - (\sqrt{3} + \sqrt{6})$\\By comparing this with the target form, we can identify the values of $A$, $B$, $C$, and $D$:\\$A = 2$, $B = 2$, $C = 3$, and $D = 6$\\We have obtained the values of $A$, $B$, $C$, and $D$ and rewritten the expression in the target form. Now, we need to verify the accuracy of our approach and conclusions.\\We will review the steps taken to ensure that the simplification and conversion to the target form are correct.\\1. We rationalized the denominator by multiplying by the conjugate, which is a standard and valid technique.\\2. We simplified the numerator and denominator correctly.\\3. We rearranged the terms and factored out a common factor, which is a valid algebraic manipulation.\\4. We compared the expression with the target form and identified the values of $A$, $B$, $C$, and $D$ correctly.\\Our approach seems to be valid, and we have not found any potential errors or inconsistencies in our reasoning.\\However, we can improve our approach by double-checking our calculations and ensuring that we have considered all possible simplifications and conversions.\\We can also evaluate our process by considering what worked well (e.g., using the conjugate to rationalize the denominator) and what could be improved (e.g., double-checking calculations). We were given the task to simplify the expression $\frac{1+\sqrt{2}}{2+\sqrt{3}}$ and convert it to the form $A(1+\sqrt{B})-(\sqrt{C}+\sqrt{D})$, where $A$, $B$, $C$, and $D$ are positive integers. We needed to find the values of $A$, $B$, $C$, and $D$ and calculate their sum.\\Our thought process involved:\\1. Assessing the problem and identifying the key elements, constraints, and objectives.\\2. Simplifying the expression by rationalizing the denominator and multiplying by the conjugate.\\3. Advancing our reasoning by performing the multiplication and simplification steps.\\4. Synthesizing the information by combining the simplified numerator and denominator and converting the expression to the target form.\\5. Verifying the accuracy of our approach and conclusions by reviewing the steps and checking for potential errors or inconsistencies.\\Our final conclusion is that the simplified expression is $2(1 + \sqrt{2}) - (\sqrt{3} + \sqrt{6})$, which corresponds to the values $A = 2$, $B = 2$, $C = 3$, and $D = 6$.\\Therefore, the sum of $A$, $B$, $C$, and $D$ is $2 + 2 + 3 + 6 = 13$.\\The final answer is: 13  \\
\bottomrule
\end{longtable}



\end{document}